\documentclass{opt2025}

\usepackage{bbm}
\usepackage{caption}

\usepackage{booktabs}
\usepackage{array}
\usepackage{xspace}
\usepackage{float}

\newcommand{\R}{\mathbb{R}}

\newcommand{\Tr}{\mathrm{Tr}}
\newcommand{\norm}[1]{\left\lVert #1 \right\rVert}

\newcommand{\Ghat}{\widehat{G}}

\newcommand{\condproxy}{\rho_{\text{cond}}}
\newcommand{\momentproxy}{\rho_{\text{moment}}}
\newcommand{\cmr}{\textsc{CMR}\xspace}

\title[Chebyshev Moment Regularization]{Chebyshev Moment Regularization (CMR): Condition-Number Control with Moment Shaping}

\optauthor{%
  \Name{Jinwoo Baek}\Email{baekji@oregonstate.edu}\\
  \addr School of Electrical Engineering and Computer Science\\
  Department of Mathematics\\
  Oregon State University\\
  Corvallis, OR 97331, USA
}

\begin{document}
\maketitle

\begin{abstract}
We introduce \textbf{Chebyshev Moment Regularization (CMR)}, a simple, architecture-agnostic loss that directly optimizes layer spectra. CMR jointly controls spectral edges via a log-condition proxy and shapes the interior via Chebyshev moments, with a decoupled, capped mixing rule that preserves task gradients. We prove strictly monotone descent for the condition proxy, bounded moment gradients, and orthogonal invariance. In an adversarial ``$\kappa$-stress'' setting (MNIST, 15-layer MLP), \emph{compared to vanilla training}, CMR reduces mean layer condition numbers by $\sim\!10^3$ (from $\approx3.9\!\times\!10^3$ to $\approx3.4$ in 5 epochs), increases average gradient magnitude, and restores test accuracy ( $\approx10\%\!\to\!\approx86\%$ ). These results support \textbf{optimization-driven spectral preconditioning}: directly steering models toward well-conditioned regimes for stable, accurate learning.
\end{abstract}

\section{Introduction}
Training very deep networks is brittle: layer spectra often turn pathological—$\sigma_{\min}$ collapses, $\sigma_{\max}$ inflates, and condition numbers explode—yielding poorly scaled gradients and stalled optimization. Common fixes such as residual connections and normalization~\citep{He2016deep,Ioffe2015batch} or careful initialization~\citep{Glorot2010understanding} help indirectly; they do not directly optimize the spectral geometry that governs numerical stability.

\noindent We propose \textbf{Chebyshev Moment Regularization (\cmr)}, a drop-in loss that treats the spectrum as a first-class training signal. \cmr couples a \emph{log-condition} proxy that targets the spectral edges ($\sigma_{\max},\sigma_{\min}$) with \emph{Chebyshev moments} of a normalized Gram operator that shape the interior distribution. A decoupled, \emph{capped mixing} rule scales spectral gradients relative to task gradients, preserving task signal while keeping spectral intervention bounded. Unlike spectral-norm regularization, which \emph{constrains} only the operator norm $\|W\|_2$~\citep{Miyato2018spectral}, or orthogonality/Parseval penalties that drive layers toward (near-)isometries~\citep{Cisse2017parseval}, \cmr provides fine-grained, orthogonally invariant control over both edges and mass with minimal implementation burden.

\paragraph{Contributions.}
\begin{itemize}
\item \textbf{Method.} \cmr is a lightweight, drop-in loss that directly reshapes layer spectra—controlling the edges and smoothing the interior—while mixing with task gradients under a small cap to preserve training dynamics; applicable across architectures and optimizers in formulation.
\item \textbf{Theory.} (a) Gradient flow on $\condproxy$ satisfies a strict descent identity (monotonic conditioning improvement); (b) moment gradients are bounded and scale-friendly; (c) the penalty is invariant under $QWR$ for orthogonal $Q,R$.
\item \textbf{Practice.} In adversarial ``$\kappa$-stress,'' \cmr reconditions layers by orders of magnitude while \emph{increasing} average gradient norms, restoring trainability and substantially improving test accuracy relative to vanilla training.
\end{itemize}

\section{Method: Chebyshev Moment Regularization}
\label{sec:method}
Let $W\in\R^{m\times n}$ be a layer and $G=W^\top W$. Define
\[
\condproxy(W) = \tfrac12\Big(\log\norm{W}_2^2 - \log\lambda_{\min}(W^\top W + \epsilon I)\Big)
= \log\sigma_{\max}(W) - \tfrac12 \log(\sigma_{\min}^2(W)+\epsilon),\label{eq:condproxy}
\]
\[
\momentproxy(W) = \sum_{k=3}^{K} w_k\, s_k(W)^2,\qquad
s_k(W)=\tfrac1n \Tr\!\big(T_k(\Ghat)\big),\label{eq:momentproxy}
\]
where $\epsilon>0$, $w_k=\exp(\beta(k-3))$, $n=\text{cols}(W)$, and $\Ghat=\frac{G-cI}{d}$ with
$c=\frac{\lambda_{\max}(G)+\lambda_{\min}(G)}{2}$ and $d=\max\{\frac{\lambda_{\max}(G)-\lambda_{\min}(G)}{2},\epsilon\}$ so that $\sigma(\Ghat)\subseteq[-1,1]$. The training objective is
\[
\mathcal{L}(\theta)
= \mathcal{L}_{\text{task}}(\theta)\;+\;\lambda \sum_{\ell=1}^L \left(\alpha_1\,\condproxy(W^{(\ell)}) + \alpha_2\,\momentproxy(W^{(\ell)})\right).
\]
\paragraph{On using $k\ge3$ moments.}
We use $k\!\ge\!3$ because $s_0,s_1,s_2$ encode mass/mean/variance of the normalized spectrum and are largely determined by the edge normalization and the condition proxy (App.~\ref{app:k3-moments}).

\paragraph{Decoupled, capped mixing.}
We backprop the task loss to get $g_{\text{task}}$, then backprop the spectral penalty to get $g_{\text{spec}}$.
Let $\rho_{\text{spec}}\!\in(0,1]$ and set
\(
\widetilde g_{\text{spec}}
= \min\!\Big\{1, \tfrac{\rho_{\text{spec}}\norm{g_{\text{task}}}}{\norm{g_{\text{spec}}}}\Big\}\, g_{\text{spec}},
\)
and update with $g=\widetilde g_{\text{spec}}+g_{\text{task}}$.
This preserves task signal while enforcing \emph{bounded spectral intervention}. We use a small stabilizer $\delta=10^{-12}$ in the cap denominator.

\begin{algorithm2e}[H]
\caption{\cmr-SGD with decoupled, capped spectral gradients}
\KwIn{model $\theta$, weights $(\lambda,\alpha_1,\alpha_2)$, cap $\rho_{\text{spec}}$, warmup $T_{\text{w}}$}
\For{steps $t=0,1,2,\dots$}{
  Compute $g_{\text{task}} \leftarrow \nabla_\theta \mathcal{L}_{\text{task}}$\;
  $\lambda_t \leftarrow \lambda \cdot \min\{1, t/T_{\text{w}}\}$\;
  Compute $g_{\text{spec}} \leftarrow \nabla_\theta \sum_\ell (\alpha_1\condproxy + \alpha_2\momentproxy)$\;
  Scale: $\gamma \leftarrow \min\{1, \rho_{\text{spec}}\norm{g_{\text{task}}}/(\norm{g_{\text{spec}}}+\delta)\}$; $\widetilde g_{\text{spec}}\!\leftarrow\! \lambda_t\gamma g_{\text{spec}}$\;
  Update with $g=\widetilde g_{\text{spec}}+g_{\text{task}}$\;
}
\end{algorithm2e}

\section{Theory (proof sketches)}
\label{sec:theory}
We show: (i) the condition proxy decreases under its gradient flow, (ii) moment gradients are bounded under mild spread, and (iii) the penalty is orthogonally invariant. Proofs are in App.~\ref{app:proofs}; the rationale for using $k\!\ge\!3$ moments is in App.~\ref{app:k3-moments}.

\paragraph{Assumptions and notation.}
We take $\epsilon>0$ so that $\lambda_{\min}(W^\top W+\epsilon I)>0$. Gradients of spectral terms are well-defined whenever the extremal singular values are simple; otherwise, statements hold almost everywhere and extend via Clarke subgradients of spectral functions~\citep{Lewis1996,Sun1988,Bhatia1997}. For Lemma~\ref{lem:moment_grad} we assume a nontrivial spectral spread $\lambda_{\max}(G)-\lambda_{\min}(G)\ge \theta\,\lambda_{\max}(G)$ for some $\theta\in(0,1]$; see Remark~\ref{rem:spread} for the $\theta\to 0$ case.

\begin{theorem}[Monotone Descent for the Condition Proxy] \label{thm:descent}
Let $W(t)$ evolve under the gradient flow $\dot{W}(t) = -\eta\nabla_W \condproxy(W(t))$ for $\eta > 0$. The condition proxy exhibits a strict descent property:
$$\frac{d}{dt}\condproxy(W(t)) = -\eta\norm{\nabla_W\condproxy(W(t))}_F^2 \le 0.$$
\end{theorem}
\textit{Proof Sketch.} The result is an exact dissipation identity derived from the chain rule: $\frac{d}{dt}\condproxy = \langle\nabla\condproxy, \dot{W}\rangle_F = \langle\nabla\condproxy, -\eta\nabla\condproxy\rangle_F$. This confirms that minimizing $\condproxy$ is a well-posed objective that directly pushes the model towards better-conditioned states. \hfill$\square$

\begin{corollary}[Control over the Log-Condition Number] \label{cor:kappa_control}
The condition proxy $\condproxy$ and the true log-condition number $\log\kappa(W)$ are related by the identity $\log\kappa(W) = \condproxy(W) + \frac{1}{2}\log(1 + \epsilon/\sigma_{\min}^2(W))$. Thus, the monotone decrease of $\condproxy$ guaranteed by Theorem~\ref{thm:descent} forces a non-increasing trend in $\log\kappa(W)$, bounded by an additive term that vanishes as $\epsilon \to 0$.
\end{corollary}

\paragraph{Discrete-step behavior.}
While Theorem~\ref{thm:descent} is stated for gradient flow, in our experiments we observe a \emph{discrete-step monotonic trend} for $\rho_{\text{cond}}$ under standard optimizers (Sec.~\ref{sec:exp}), supporting its use as a direct conditioning signal.

\noindent\emph{Note.} The identity holds pointwise; the monotonicity conclusion applies almost everywhere along trajectories where $\rho_{\text{cond}}$ is differentiable, and extends in the sense of energy dissipation inequalities using Clarke subgradients~\citep{Lewis1996}.

\begin{lemma}[Moment Gradients are Bounded and Scale-Friendly] \label{lem:moment_grad}
Assume the spectral spread $\lambda_{\max}(G) - \lambda_{\min}(G) \ge \theta\,\lambda_{\max}(G)$ for some $\theta\in(0,1]$. Then the Frobenius norm of the moment penalty's gradient satisfies
\[
\norm{\nabla_W \momentproxy}_F \le \frac{C \cdot K}{\norm{W}_2} + \mathcal{O}(\norm{W}_2^{-3}).
\]
Moreover, without the spread assumption one always has the general bound
\[
\|\nabla_W \rho_{\text{moment}}\|_F \;\le\; \frac{C'}{\sqrt{n}\,\epsilon}\,K\,\|W\|_2,
\]
so the claimed $1/\|W\|_2$ decay is precisely the regime where the affine normalization is edge–dominated (i.e., $d=\Theta(\|W\|_2^2)$).
\end{lemma}
\textit{Proof Sketch.} The proof proceeds in three steps. (1) The gradient of the trace term is $\nabla_{\Ghat}\Tr(T_k(\Ghat)) = T_k'(\Ghat)$. (2) The derivative of a Chebyshev polynomial satisfies $T_k'(x)=k\,U_{k-1}(x)$ with $\sup_{x\in[-1,1]}|U_{k-1}(x)|=k$, hence $\|T_k'(\cdot)\|_\infty\le k^2$~\citep[Thm.~1.2]{Rivlin1990}. (3) The affine scaling denominator $d = \max\{\frac{\lambda_{\max}-\lambda_{\min}}{2},\epsilon\}$ scales as $\Theta(\norm{W}_2^2)$ under the spread assumption. Combining these facts via the chain rule for matrix calculus ($dG = W^\top dW + dW^\top W$) yields the stated $1/\norm{W}_2$ decay; terms involving gradients of $c,d$ are lower order ($\mathcal{O}(d^{-2})$). The alternative bound follows by taking $d=\epsilon$. \hfill$\square$

\begin{remark}\label{rem:spread}
When the spectrum is nearly degenerate (spread $\to 0$), $d$ is set by $\epsilon$ and the bound becomes linear in $\|W\|_2$. In practice we keep $\epsilon$ small and rely on the condition proxy to widen the edges, quickly entering the favorable $1/\|W\|_2$ regime.
\end{remark}

\begin{proposition}[Orthogonal Invariance] \label{prop:invariance_main}
The \cmr penalty is invariant under orthogonal transformations. For any orthogonal matrices $Q, R$ of appropriate dimensions,
$$\condproxy(QWR) = \condproxy(W) \quad \text{and} \quad \momentproxy(QWR) = \momentproxy(W).$$
\end{proposition}
\textit{Proof Sketch.} The singular values of $W$ are invariant to orthogonal transformations, which proves invariance of $\condproxy$. For the moment term, the Gram matrix transforms as $G_{QWR} = R^\top G_W R$. By the functional calculus for matrix polynomials~\citep[Chap.~1]{Higham2008}, $T_k(R^\top \Ghat_W R)=R^\top T_k(\Ghat_W)R$. Using the cyclic property of the trace, $s_k(QWR)=\tfrac1n\Tr(T_k(\Ghat_{QWR}))=\tfrac1n\Tr(T_k(\Ghat_W))=s_k(W)$. Thus $\rho_{\text{moment}}$ and hence the full penalty are orthogonally invariant. \hfill$\square$

\section{Why Well-Conditioned Layers Help}
\paragraph{Notation.}
Let $f: \mathbb{R}^{d_0} \to \mathbb{R}^{d_L}$ be a depth-$L$ feedforward neural network defined by
\[
f(x) = W_L z_{L-1}, \quad \text{where} \quad
z_\ell = \phi_\ell(h_\ell), \quad h_\ell = W_\ell z_{\ell-1}, \quad z_0 = x.
\]
Each layer consists of an affine map $W_\ell$ followed by an elementwise nonlinearity $\phi_\ell$.
Define the layerwise Jacobian factor $J_\ell(x) := \frac{\partial z_\ell}{\partial z_{\ell-1}}$.

Since $\phi_\ell$ acts elementwise, its Jacobian is diagonal:
\[
D_\ell(x) := \mathrm{diag}(\phi'_\ell(h_\ell(x))) \in \mathbb{R}^{d_\ell \times d_\ell}.
\]
Thus, by the chain rule,
\[
J_\ell(x) = D_\ell(x)\,W_\ell.
\]
For notational convenience, we denote the layerwise Jacobian factor in the alternate order
\[
J_\ell(x) := W_\ell D_\ell(x),
\]
so that the full network Jacobian is written as
\[
J(x) := \frac{\partial f(x)}{\partial x}
= J_L(x) \cdots J_1(x),
\]
with the convention that nonlinearities appear on the right within each $J_\ell$.

\paragraph{Assumption.}
Assume $\sigma_{\min}(W_\ell) > 0$ for all layers $\ell$.
Furthermore, assume there exist deterministic constants $\mu_\ell, L_\ell$ such that
\[
0 \le \mu_\ell \le \|D_\ell(x)\|_2 \le L_\ell < \infty
\quad \text{for all } x \in \mathcal{X},
\]
where $\mathcal{X}$ denotes the data region of interest (e.g., the training domain).

\begin{proposition} [Layerwise Jacobian bound.]
\label{prop:layerwise}
For every input $x$,
\[
\sigma_{\max}(J(x)) \le \prod_{\ell=1}^L L_\ell \sigma_{\max}(W_\ell), \qquad \sigma_{\min}(J(x)) \le \prod_{\ell=1}^L\mu_\ell\sigma_{\min}(W_\ell),
\]
hence,
\[
\boxed{\kappa\left(J(x)\right)\ \le\
\left(\prod_{\ell=1}^{L}\kappa(W_\ell)\right)\cdot
\left(\prod_{\ell=1}^{L}\tfrac{L_\ell}{\mu_\ell}\right)}.
\]
\end{proposition}

\begin{proof}
Let $J(x) = \frac{\partial f(x)}{\partial x} = J_L(x)J_{L-1}(x)\cdots J_1(x)$ be the Jacobian of $f$ with respect to the input $x$. For any pair of matrices $A$, $B$, the singular values satisfy the standard inequalities:
\[
\sigma_{\max}(AB)\le\sigma_{\max}(A)\,\sigma_{\max}(B),
\qquad
\sigma_{\min}(AB)\ge\sigma_{\min}(A)\,\sigma_{\min}(B).
\tag{$*$}
\]
\begin{enumerate}
\item For each layer $\ell$, by assumption,
\[
\mu_\ell \le \|D_\ell(x)\|_2 = \max_i|\phi'(h_\ell^{(i)}(x))| \le L_\ell.
\]
Hence, 
\[
\sigma_{\max}(J_\ell(x))
\le \sigma_{\max}(W_\ell)\,\|D_\ell(x)\|_2
\le L_\ell\sigma_{\max}(W_\ell),
\]
\[
\sigma_{\min}(J_\ell(x))
\ge \sigma_{\min}(W_\ell)\|D_\ell(x)^{-1}\|_2^{-1}
\ge \mu_\ell\,\sigma_{\min}(W_\ell).
\]
\item Applying inequality $(*)$ recursively to the product $J(x) = J_L \cdots J_1$, we obtain
\[
\sigma_{\max}(J(x)) \le \prod_{\ell=1}^L \sigma_{\max}(J_\ell(x)) \le \prod_{\ell=1}^L L_\ell\sigma_{\max}(W_\ell),
\]
\[
\sigma_{\min}(J(x)) \ge \prod_{\ell=1}^L \sigma_{\min}(J_\ell(x)) \ge \prod_{\ell=1}^L\mu_{\ell}\sigma_{\min}(W_\ell)
\]
\item Taking the ratio yields
\[
\kappa(J(x))
=\frac{\sigma_{\max}(J(x))}{\sigma_{\min}(J(x))}
\le
\left(\prod_{\ell=1}^{L}\frac{\sigma_{\max}(W_\ell)}{\sigma_{\min}(W_\ell)}\right)
\left(\prod_{\ell=1}^{L}\frac{L_\ell}{\mu_\ell}\right)
=
\left(\prod_{\ell=1}^{L}\kappa(W_\ell)\right)
\left(\prod_{\ell=1}^{L}\tfrac{L_\ell}{\mu_\ell}\right).
\]
\end{enumerate}
This completes the proof.
\end{proof}

\paragraph{Implication of Proposition~\ref{prop:layerwise}.}
This bound shows that the network Jacobian condition number $\kappa(J(x))$
grows multiplicatively with layerwise condition numbers $\kappa(W_\ell)$
and activation slope ratios $L_\ell / \mu_\ell$.
Well-conditioned layers therefore directly improve the overall Jacobian conditioning,
which is essential for preserving signal geometry and preventing forward/backward distortion.

\begin{corollary}[Gradient propagation]
\label{cor:grad}
Backpropagation satisfies:
\[
\left\| \frac{\partial \mathcal{L}}{\partial h_\ell} \right\|
\le \prod_{i > \ell} \sigma_{\max}(J_i(x)) \cdot
\left\| \frac{\partial \mathcal{L}}{\partial h_L} \right\|.
\]
By Proposition 1, decreasing $\kappa(W_i)$ reduces these factors and mitigates explosion or vanishing.

\end{corollary}

\begin{proof}
We prove the result in three steps:
\begin{enumerate}
\item Let $h_\ell$ denote the preactivation at layer $\ell$ and $\mathcal L$ the scalar loss. By the chain rule of differentiation, the gradient of the loss with respect to $h_\ell$ is given by
\[
\frac{\partial\mathcal L}{\partial h_\ell} = J_{\ell+1}(x)^\top J_{\ell+2}(x)^\top \cdots J_L(x)^\top\frac{\partial\mathcal L}{\partial h_L}.
\]
That is, each gradient at layer $\ell$ is obtained by repeatedly multiplying the Jacobian transposes of all subsequent layer.
\item Taking the spectral norm $\|\cdot\|_2$ on both sides and using the submultiplicativity property $\|AB\| \le \|A\|\|B\|$, we have:
\[
\left\|\frac{\partial\mathcal L}{\partial h_\ell}\right\| \le \left(\prod_{i=\ell+1}^L\|J_i(x)^\top\|\right)\left\|\frac{\partial\mathcal L}{\partial h_L}\right\|.
\]
Since the spectral norm is invariant under transposition ($\|A^\top\| = \|A\|$ for all matrices $A$), this simplifies to
\[
\left\|\frac{\partial\mathcal L}{\partial h_\ell}\right\| \le \left(\prod_{i > \ell}\|J_i(x)\|\right)\left\|\frac{\partial\mathcal{L}}{\partial h_L}\right\|.
\]
\item By proposition 1, each $\|J_i(x)\|$ can be upper-bounded as $\|J_i(x)\| \le L_i\sigma_{\max}(W_i)$. Therefore, reducing $\kappa(W_i)$ Which directly reduces $\sigma_{\max}(W_i)$ relative to $\sigma_{\min}(W_i)$ tightens the bound on $\|J_i(x)\|$ and hence on the product $\prod_{i>\ell}\|J_i(x)\|$.
\end{enumerate}
\end{proof}

\paragraph{Implication of Corollary~\ref{cor:grad}.}
This result reveals how backpropagated gradients grow or decay across layers,
with their norms governed by the product of Jacobian operator norms $\|J_i(x)\|$.
Since these norms depend on the spectral properties of $W_i$,
reducing $\kappa(W_i)$ suppresses gradient explosion or vanishing,
leading to more stable and reliable training dynamics.

\section{Experiment: Adversarial Ill-Conditioning ($\kappa$-stress)}
\label{sec:exp}

\paragraph{Setup.}
\begin{itemize}
    \item \textbf{Data/Model.} MNIST; a deep 15-layer MLP (width 256) with \texttt{tanh} activations.
    \item \textbf{$\kappa$-stress.} We use an orthogonal initialization scaled by $0.06$ to create an intentionally adversarial, ill-conditioned starting point where vanilla training struggles.
    \item \textbf{Optimizer.} Adam with a learning rate of $10^{-3}$.
    \item \textbf{\cmr.} We set $(K,\lambda,\alpha_1,\alpha_2,\beta)=(5,0.02,1.0,0.1,0.15)$, using a 2-epoch warmup, a spectral gradient cap $\rho_{\text{spec}}=0.5$, and a global gradient clip of $5.0$.
    \item \textbf{Metrics.} Test accuracy, average per-step gradient norm (global $\ell_2$ over parameters, averaged across steps within an epoch), mean layer-wise condition number $\kappa(W)$ \emph{(arithmetic mean across layers)}, and the maximum absolute Chebyshev moment $\max_{k\in[3,K]}|s_k|$.
    \item \textbf{Baselines.} Vanilla corresponds to $\lambda{=}0$; all other settings (optimizer, clipping, warmup) are identical between vanilla and \cmr.
\end{itemize}

\paragraph{Results and Analysis.}
\cmr acts as a powerful spectral preconditioner, successfully restoring the trainability of the network. The headline result is the dramatic improvement in conditioning: \cmr collapses the mean layer condition number $\kappa(W)$ from a near-singular $\sim\!3.9\times 10^3$ to a well-behaved $3.4$ within five epochs—a thousand-fold improvement (Fig.~\ref{fig:all} (c)).

\noindent This drastic re-conditioning has a profound impact on optimization. While the vanilla model remains stuck at $\approx\!10\%$ accuracy, the \cmr-regularized model rapidly recovers to $\approx\!86\%$ (Fig.~\ref{fig:all} (a)). Crucially, this recovery is accompanied not by smaller, but by \emph{larger and more effective} gradients; the average gradient norm rises from a stagnant $\approx\!1.4$ to a healthy $\approx\!4.2$ (Fig.~\ref{fig:all} (b), y-axis on a logarithmic scale). \cmr lifts $\sigma_{\min}$ while moderating $\sigma_{\max}$ to preserve Jacobian signal, and the moment term bounds higher-order moments to smooth the spectrum (Fig.~\ref{fig:all} (d)).
\footnote{We characterize this failure mode by its \emph{spectral signature} (e.g., extreme $\kappa$) rather than the symptom of “vanishing gradients,” which can arise from multiple mechanisms.} See Appendix~\ref{app:standard_baselines} for standard (non–$\kappa$-stress) L2/SN comparisons; this setup differs from the main experiment.

\begin{figure}[H]
\centering
\begin{minipage}[t]{0.45\textwidth}
    \centering
    \includegraphics[width=\linewidth]{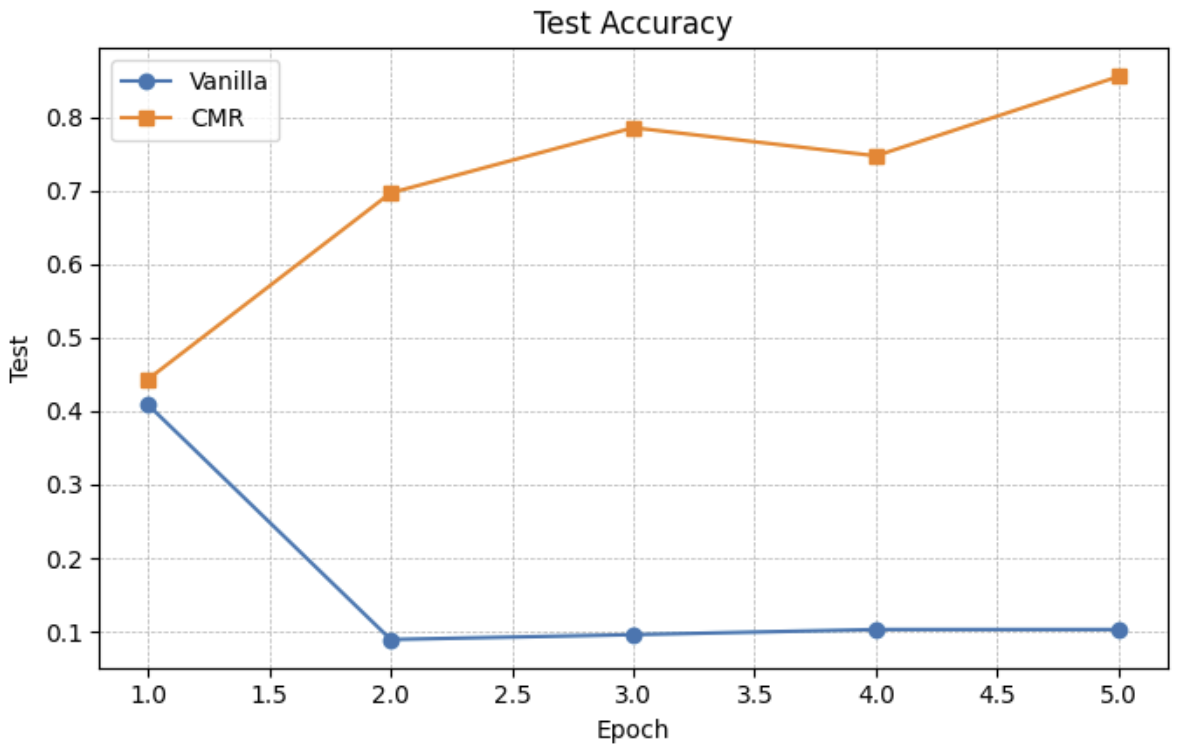}
    \caption*{{\footnotesize (a) Test accuracy}}
\end{minipage}\hfill
\begin{minipage}[t]{0.45\textwidth}
    \centering
    \includegraphics[width=\linewidth]{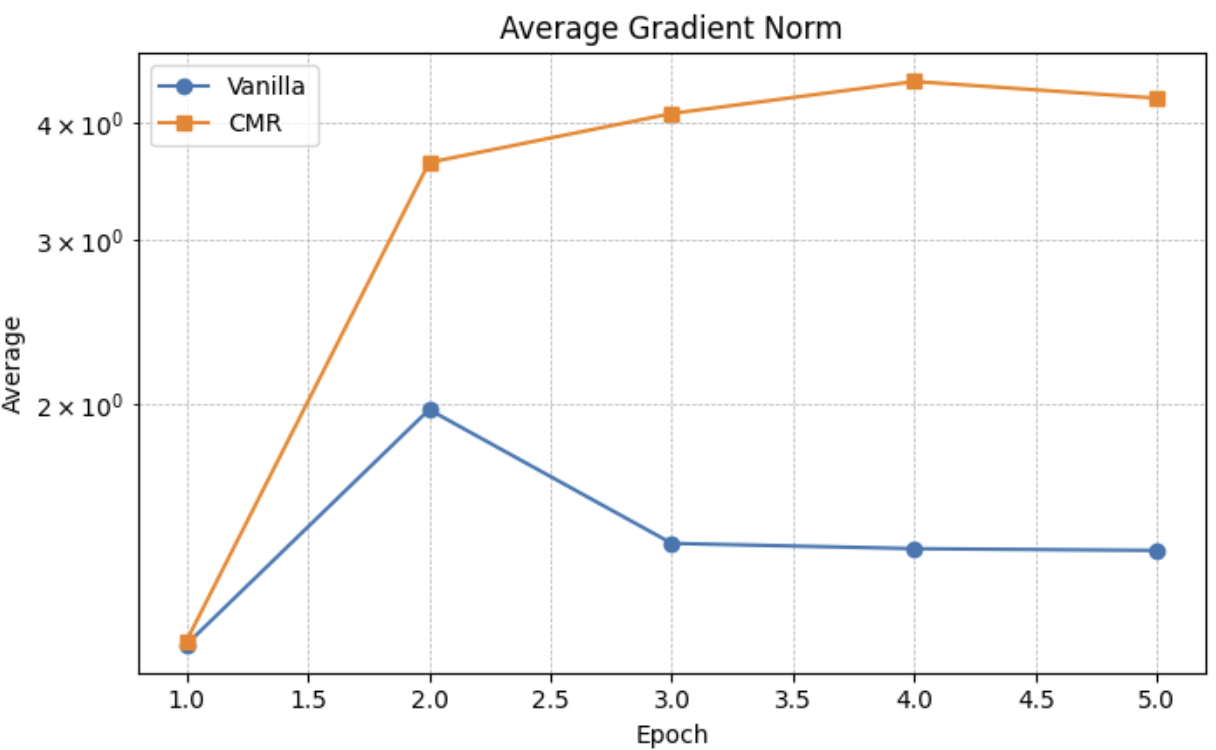}
    \caption*{{\footnotesize (b) Avg. grad norm (\textit{logarithmic y-axis})}}
\end{minipage}

\vspace{4pt} 

\begin{minipage}[t]{0.45\textwidth}
    \centering
    \includegraphics[width=\linewidth]{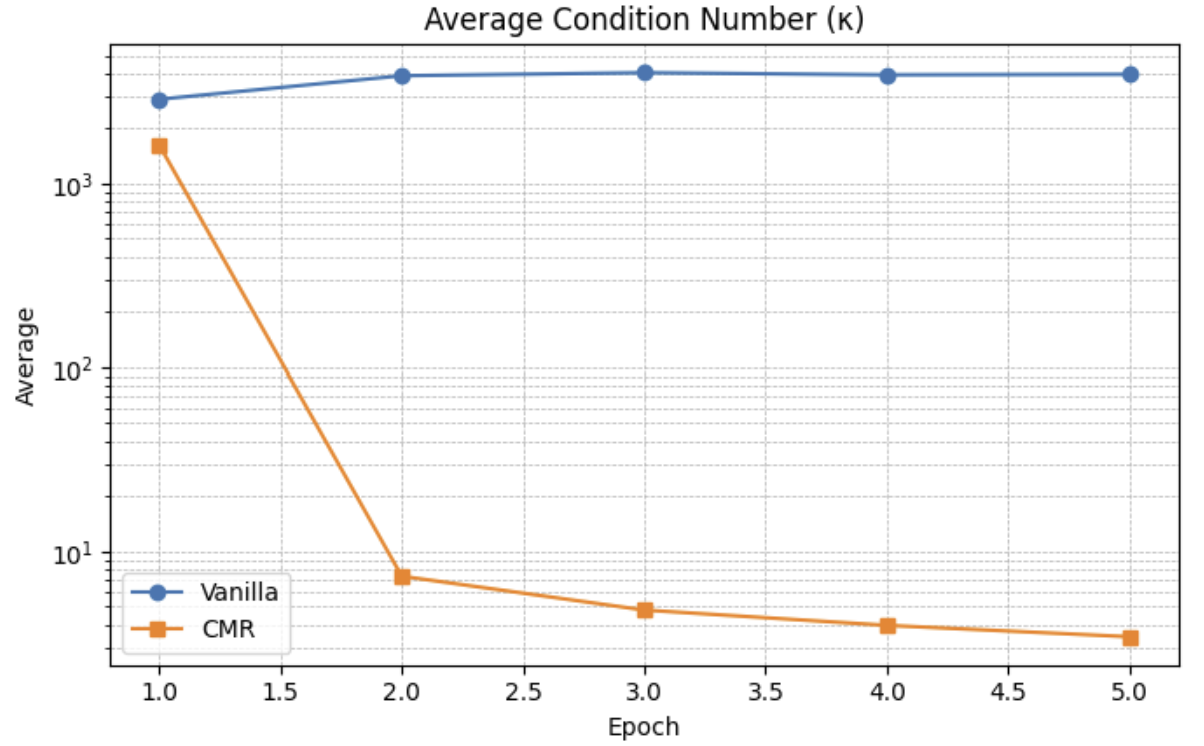}
    \caption*{{\footnotesize (c) Mean $\kappa(W)$ (\textit{logarithmic y-axis})}}
\end{minipage}\hfill
\begin{minipage}[t]{0.45\textwidth}
    \centering
    \includegraphics[width=\linewidth]{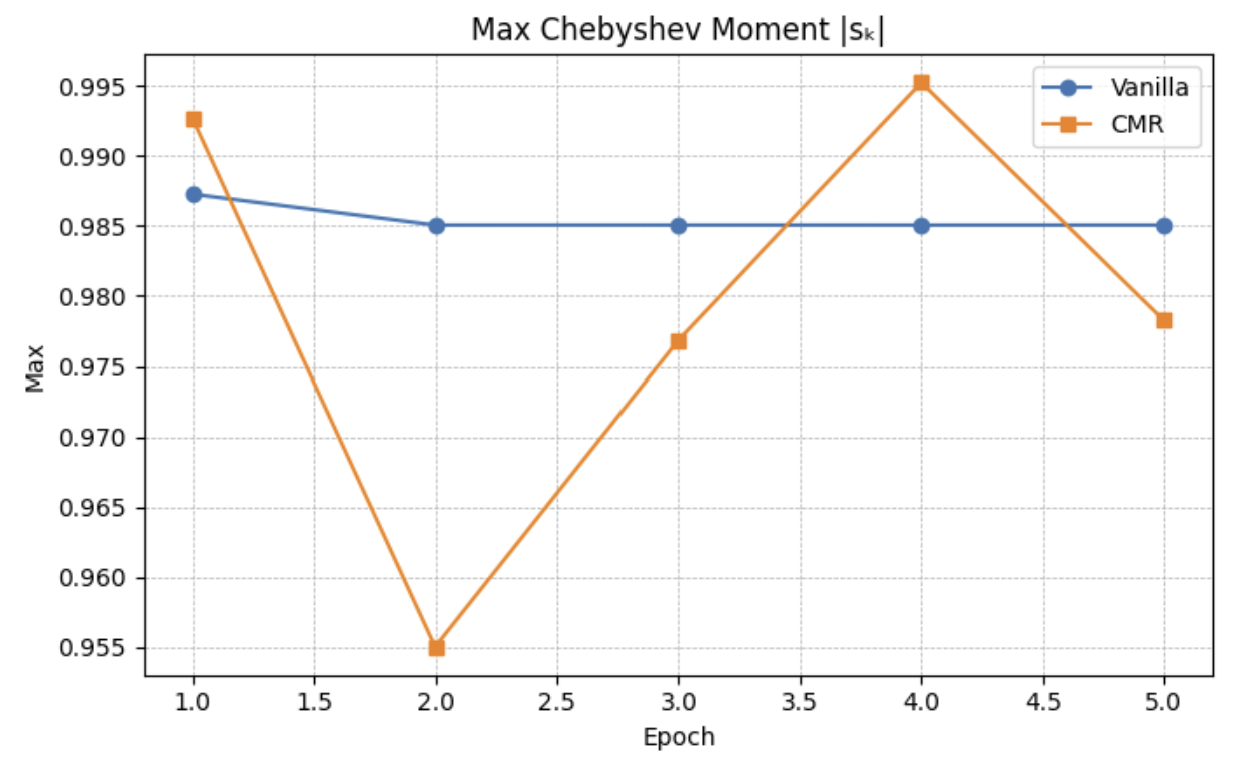}
    \caption*{{\footnotesize (d) $\max_{k\in[3,5]}|s_k|$}}
\end{minipage}

\caption{\textbf{CMR under $\kappa$-stress.} CMR (orange) restores trainability and directly improves spectral conditioning vs. vanilla (blue).}
\label{fig:all}
\end{figure}

\paragraph{Takeaway.}
In this regime, the core driver of trainability is \emph{conditioning}, not raw gradient magnitude. \cmr functions as an \emph{optimization-driven preconditioner}: the condition proxy sculpts the spectral edges by lifting $\sigma_{\min}$ and moderating $\sigma_{\max}$, while the moment term smooths interior mass. This holistic control lowers $\log\kappa$, producing gradients that are larger, better-scaled, and better aligned with useful learning directions—enabling the sharp recovery in test performance.

\section{Related Work}
\label{sec:related_work}
Stability has been pursued via \textbf{architecture} (residual, normalization)~\citep{He2016deep,Ioffe2015batch} and \textbf{initialization}~\citep{Glorot2010understanding}, which help indirectly rather than optimizing spectra. Among \textbf{regularizers}, spectral norm constrains only $\|W\|_2$~\citep{Miyato2018spectral}, while orthogonality pushes near-isometries~\citep{Cisse2017parseval}. In contrast, \cmr is a loss-level, orthogonally invariant regularizer that jointly controls spectral edges (log-condition) and higher-order shape (Chebyshev) with a capped mixing rule, offering a more direct handle on conditioning with descent and bounded-gradient guarantees.

\section{Conclusion}
\label{sec:conclusion}
Future work includes applying \cmr to Transformers, learning which Chebyshev moments (and weights) to penalize, and running ablations that disentangle the effects of the condition proxy and the moment term.

\newpage

\bibliography{references}

\newpage

\appendix

\section{Why Moments Start at $k=3$}
\label{app:k3-moments}
Let $G=W^\top W$ with eigenvalues $\{\lambda_i\}_{i=1}^n$ and define the affine normalization
$\widehat G=(G-cI)/d$ with $c=\tfrac{1}{2}(\lambda_{\max}+\lambda_{\min})$ and
$d=\max\{\tfrac{1}{2}(\lambda_{\max}-\lambda_{\min}),\epsilon\}$ so that $\sigma(\widehat G)\subseteq[-1,1]$.
Chebyshev moments are $s_k=\frac{1}{n}\Tr(T_k(\widehat G))$ with $T_0(x)=1$, $T_1(x)=x$, and $T_2(x)=2x^2-1$.
They satisfy
\[
s_0 \equiv 1,\qquad
s_1=\tfrac{1}{n}\Tr(\widehat G)=\tfrac{\overline{\lambda}-c}{d},\qquad
s_2=\tfrac{1}{n}\Tr(2\widehat G^2-I)=2\cdot \tfrac{1}{n}\sum_i \widehat\lambda_i^2 - 1,
\]
hence $\mathrm{Var}(\widehat\lambda)=\big(\tfrac{s_2+1}{2}\big)-s_1^2$.
Thus $s_0,s_1,s_2$ encode mass/mean/variance of the \emph{normalized} spectrum—quantities already fixed by the edge-based normalization $(c,d)$ and largely governed by the condition proxy.
Penalizing $k\le2$ would double-count edge/scale control and can interfere with $\rho_{\text{cond}}$.
We therefore use $k\ge3$ to isolate higher-order shape (tails, asymmetry, peaky structure), complementing edge control without redundancy.

\section{Full Theoretical Results and Proofs} \label{app:proofs}
This appendix provides detailed derivations and proofs for the theoretical claims made in Section~\ref{sec:theory}. We adopt the notation from the main text.

\subsection{Gradient of the Condition Proxy}
To prove our main results, we first require the explicit form of the gradient for $\rho_{\text{cond}}(W)$.
We state the formula under simplicity of the extremal singular values; otherwise, subgradients exist and the identities hold almost everywhere~\citep{Lewis1996,Bhatia1997}.

\begin{lemma}[Gradient of $\rho_{\text{cond}}$]
\label{lem:grad_form_app}
Let $W=U\Sigma V^\top$ be the singular value decomposition of $W$. Assume the largest singular value $\sigma_{\max}(W)$ and smallest singular value $\sigma_{\min}(W)$ are simple. Let $(u_1, v_1)$ and $(u_r, v_r)$ be the corresponding pairs of left and right singular vectors. The gradient of the condition proxy is given by:
\[
\nabla_W \rho_{\mathrm{cond}}(W) = \frac{1}{\sigma_{\max}(W)} u_1 v_1^\top - \frac{\sigma_{\min}(W)}{\sigma_{\min}(W)^2 + \epsilon} \, u_r v_r^\top.
\]
Consequently, the squared Frobenius norm of the gradient is:
\[
\norm{\nabla_W \rho_{\mathrm{cond}}(W)}_F^2 = \frac{1}{\sigma_{\max}(W)^2} + \left(\frac{\sigma_{\min}(W)}{\sigma_{\min}(W)^2 + \epsilon}\right)^2.
\]
\end{lemma}
\begin{proof}
The condition proxy is defined as $\rho_{\mathrm{cond}}(W) = \log\sigma_{\max}(W) - \frac{1}{2}\log(\sigma_{\min}(W)^2 + \epsilon)$. The gradient of a simple singular value $\sigma_i(W)$ with corresponding vectors $(u_i, v_i)$ is the rank-one matrix $\nabla_W \sigma_i(W) = u_i v_i^\top$. Applying the chain rule yields:
\begin{align*}
\nabla_W \log\sigma_{\max}(W) &= \frac{1}{\sigma_{\max}(W)} \nabla_W \sigma_{\max}(W) = \frac{1}{\sigma_{\max}(W)} u_1 v_1^\top, \\
\nabla_W \left(\tfrac{1}{2}\log(\sigma_{\min}(W)^2 + \epsilon)\right) &= \frac{1}{2} \frac{1}{\sigma_{\min}(W)^2 + \epsilon} \nabla_W(\sigma_{\min}(W)^2) \\
&= \frac{2\sigma_{\min}(W)}{2(\sigma_{\min}(W)^2 + \epsilon)} \nabla_W \sigma_{\min}(W) = \frac{\sigma_{\min}(W)}{\sigma_{\min}(W)^2 + \epsilon} u_r v_r^\top.
\end{align*}
Subtracting the second term from the first gives the gradient formula. For the Frobenius norm, we use the fact that singular vectors form orthonormal sets, meaning $\langle u_1v_1^\top, u_rv_r^\top \rangle_F = \Tr(v_1u_1^\top u_rv_r^\top) = (u_1^\top u_r)(v_1^\top v_r) = 0$ for $1 \neq r$. Thus, the squared norm is the sum of the squared norms of the two orthogonal rank-one components:
\[
\norm{\nabla_W \rho_{\mathrm{cond}}(W)}_F^2 = \norm{\tfrac{1}{\sigma_{\max}} u_1 v_1^\top}_F^2 + \norm{\tfrac{\sigma_{\min}}{\sigma_{\min}^2+\epsilon} u_r v_r^\top}_F^2 = \tfrac{1}{\sigma_{\max}^2}\norm{u_1 v_1^\top}_F^2 + \left(\tfrac{\sigma_{\min}}{\sigma_{\min}^2+\epsilon}\right)^2\norm{u_r v_r^\top}_F^2.
\]
Since $\norm{u_iv_i^\top}_F^2 = \norm{u_i}_2^2\norm{v_i}_2^2 = 1$, the result follows.
\end{proof}

\subsection{Proof of Theorem~\ref{thm:descent} and Corollary~\ref{cor:kappa_control}}
\begin{proof}[Proof of Theorem~\ref{thm:descent}]
We analyze the time derivative of $\rho_{\text{cond}}(W(t))$ along the gradient flow $\dot{W}(t) = -\eta\nabla_W \rho_{\text{cond}}(W(t))$. By the chain rule:
\[
\frac{d}{dt}\rho_{\text{cond}}(W(t)) = \left\langle \nabla_W \rho_{\text{cond}}(W(t)), \dot{W}(t) \right\rangle_F.
\]
Substituting the definition of the gradient flow:
\[
\frac{d}{dt}\rho_{\text{cond}}(W(t)) = \left\langle \nabla_W \rho_{\text{cond}}(W(t)), -\eta\nabla_W \rho_{\text{cond}}(W(t)) \right\rangle_F = -\eta \norm{\nabla_W \rho_{\text{cond}}(W(t))}_F^2.
\]
Since the squared Frobenius norm is always non-negative, we have $\frac{d}{dt}\rho_{\text{cond}}(W(t)) \le 0$. The descent is strict whenever $\nabla\rho_{\text{cond}}\neq 0$ (the generic case).
\end{proof}

\begin{proof}[Proof of Corollary~\ref{cor:kappa_control}]
The log-condition number is $\log\kappa(W) = \log\sigma_{\max}(W) - \log\sigma_{\min}(W)$. The condition proxy is $\rho_{\text{cond}}(W) = \log\sigma_{\max}(W) - \frac{1}{2}\log(\sigma_{\min}(W)^2 + \epsilon)$. We can write:
\begin{align*}
\log\kappa(W) &= \left(\log\sigma_{\max} - \tfrac{1}{2}\log(\sigma_{\min}^2+\epsilon)\right) + \tfrac{1}{2}\log(\sigma_{\min}^2+\epsilon) - \log\sigma_{\min} \\
&= \rho_{\text{cond}}(W) + \tfrac{1}{2}\log(\sigma_{\min}^2+\epsilon) - \tfrac{1}{2}\log(\sigma_{\min}^2) \\
&= \rho_{\text{cond}}(W) + \tfrac{1}{2}\log\left(\frac{\sigma_{\min}^2+\epsilon}{\sigma_{\min}^2}\right) \\
&= \rho_{\text{cond}}(W) + \tfrac{1}{2}\log\left(1 + \frac{\epsilon}{\sigma_{\min}(W)^2}\right).
\end{align*}
This is the stated identity. Since $\rho_{\text{cond}}(W)$ decreases monotonically under the flow (Theorem~\ref{thm:descent}), $\log\kappa(W)$ must also follow a non-increasing trend, perturbed only by the additive term which is positive and depends on the ratio $\epsilon/\sigma_{\min}^2$.
\end{proof}

\subsection{Proof of Lemma~\ref{lem:moment_grad} (Moment Gradient Bounds)}
\begin{proof}
The moment penalty is $\rho_{\text{moment}}(W) = \sum_{k=3}^{K} w_k s_k(W)^2$. Its gradient is $\nabla_W \rho_{\text{moment}} = \sum_{k=3}^{K} 2w_k s_k \nabla_W s_k$. We focus on bounding $\norm{\nabla_W s_k}_F$. Recall $s_k = \frac{1}{n}\Tr(T_k(\Ghat))$ where $\Ghat=(G-cI)/d$ and $G=W^\top W$.
Using the chain rule for matrix derivatives: $\nabla_W s_k = 2W \nabla_G s_k$, and $\nabla_G s_k = \frac{\partial s_k}{\partial \Ghat} \frac{\partial \Ghat}{\partial G}$.
The main term is $\frac{\partial \Tr(T_k(\Ghat))}{\partial \Ghat} = T_k'(\Ghat)$; by $T_k'(x)=k\,U_{k-1}(x)$ and $\sup_{x\in[-1,1]}|U_{k-1}(x)|=k$, we have $\|T_k'(\cdot)\|_\infty\le k^2$~\citep[Thm.~1.2]{Rivlin1990}.
Thus the dominant part of the gradient is
\[
\nabla_W s_k \approx \frac{2}{nd}\,W\, T_k'(\Ghat).
\]
Taking the Frobenius norm and using $\|AB\|_F \le \|A\|_2\|B\|_F$ gives
\[
\norm{\nabla_W s_k}_F \lesssim \frac{2}{nd} \norm{W}_2 \norm{T_k'(\Ghat)}_F
\le \frac{2}{nd} \norm{W}_2 \sqrt{n}\,\|T_k'(\Ghat)\|_2
\le \frac{2 k^2}{\sqrt{n} d} \norm{W}_2.
\]
Under the spread assumption, $d = \Theta(\|W\|_2^2)$, yielding $\|\nabla_W s_k\|_F=\mathcal{O}(k^2/\|W\|_2)$. Summing over $k=3,\dots,K$ and absorbing constants gives the stated bound; terms from $\nabla c,\nabla d$ scale as $d^{-2}$ and become $\mathcal{O}(\|W\|_2^{-3})$ after multiplying by $W$. If the spread is negligible, $d=\epsilon$ and the alternative bound follows.
\end{proof}

\subsection{Proof of Proposition~\ref{prop:invariance_main} (Orthogonal Invariance)}
\begin{proof}
Let $Q, R$ be orthogonal matrices.
\begin{enumerate}
    \item \textbf{Condition Proxy $\rho_{\text{cond}}$:} The singular values of a matrix $W$ are defined from the eigenvalues of $W^\top W$. The singular values of $QWR$ are defined from the eigenvalues of $(QWR)^\top(QWR) = R^\top W^\top Q^\top Q W R = R^\top(W^\top W)R$. Since $W^\top W$ and $R^\top(W^\top W)R$ are related by an orthogonal similarity transformation, they have the same eigenvalues. Therefore, $W$ and $QWR$ have the same singular values. Since $\rho_{\text{cond}}(W)$ depends only on $\sigma_{\max}(W)$ and $\sigma_{\min}(W)$, it follows that $\rho_{\text{cond}}(QWR)=\rho_{\text{cond}}(W)$.
    \item \textbf{Moment Proxy $\rho_{\text{moment}}$:} Let $G_W = W^\top W$ and $G_{QWR} = (QWR)^\top(QWR) = R^\top G_W R$. As shown above, $G_W$ and $G_{QWR}$ have the same set of eigenvalues.
    The affine normalization constants $c = \frac{1}{2}(\lambda_{\max}+\lambda_{\min})$ and $d = \max\{\frac{1}{2}(\lambda_{\max}-\lambda_{\min}),\epsilon\}$ depend only on the extremal eigenvalues, and are thus identical for $G_W$ and $G_{QWR}$. Let's call them $c$ and $d$.
    The normalized Gram matrices are $\Ghat_W = (G_W - cI)/d$ and $\Ghat_{QWR} = (G_{QWR} - cI)/d = (R^\top G_W R - cR^\top I R)/d = R^\top (G_W - cI) R / d = R^\top \Ghat_W R$.
    By the functional calculus for matrix polynomials~\citep[Chap.~1]{Higham2008}, $T_k(R^\top \Ghat_W R) = R^\top T_k(\Ghat_W) R$.
    Using the cyclic property of the trace,
    \[
    s_k(QWR) = \frac{1}{n}\Tr(R^\top T_k(\Ghat_W) R) = \frac{1}{n}\Tr(R R^\top T_k(\Ghat_W)) = \frac{1}{n}\Tr(T_k(\Ghat_W)) = s_k(W).
    \]
    Since the moments $s_k$ are invariant, the penalty $\rho_{\text{moment}}(W) = \sum_{k=3}^K w_k s_k(W)^2$ is also invariant.
\end{enumerate}
This completes the proof that the entire \cmr penalty is orthogonally invariant.
\end{proof}
\section{Comparison with Standard Regularizers}
\label{app:standard_baselines}

To demonstrate the general efficacy of \cmr beyond the extreme "$\kappa$-stress" scenario presented in the main text, we conduct a comparative study against standard regularization techniques: L2 regularization and Spectral Norm (SN) regularization. These experiments are performed in a standard training setup, using a 15-layer MLP with \texttt{tanh} activations, initialized with Glorot uniform weights (not the adversarial scaling). Here $\kappa(G)$ denotes the Gram-matrix condition number; since $G=W^\top W$, $\kappa(G)=\kappa(W)^2$ (up to the small numerical $\epsilon$ used for stability). Trends mirror those in the main text.

\begin{figure}[H]
    \centering
    \includegraphics[width=0.6\linewidth]{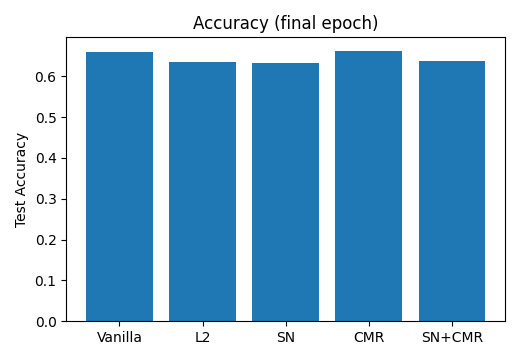} 
    \caption{\textbf{Final Test Accuracy Across Baselines.} This bar chart illustrates the final test accuracy achieved by different regularization methods after 20 training epochs. \cmr achieves competitive accuracy, slightly outperforming vanilla and matching or exceeding L2 and SN. This indicates that the spectral benefits of \cmr are gained without compromising the model's primary task performance.}
    \label{fig:acc_comparison_appendix}
\end{figure}

\paragraph{Analysis of Final Accuracy.}
Figure~\ref{fig:acc_comparison_appendix} shows that \cmr, whether applied alone or in conjunction with Spectral Norm (SN+CMR), yields a final test accuracy that is comparable to or slightly better than the vanilla baseline, L2, and SN regularization. This is a crucial finding, as it demonstrates that \cmr successfully improves the spectral properties of the network without incurring any performance penalty on the primary classification task. The consistent performance across these methods suggests that in this non-adversarial setting, accuracy alone may not fully reflect the underlying health of the optimization.

\vspace{1cm}

\begin{figure}[H]
    \centering
    \includegraphics[width=0.8\linewidth]{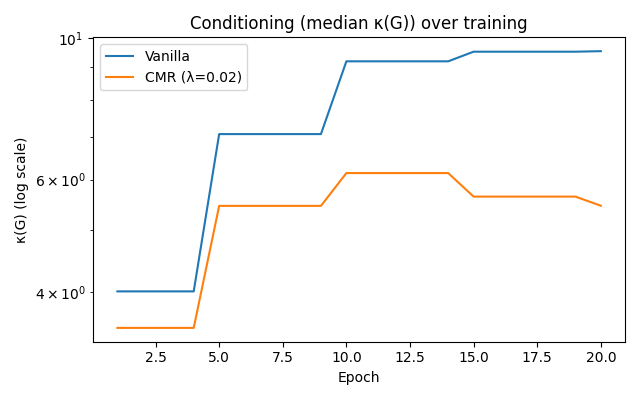} 
    \caption{\textbf{Evolution of Median Layer Condition Number ($\kappa(G)$) Over Training.} The median condition number of the Gram matrices for each layer is plotted over 20 epochs on a logarithmic scale. The vanilla model's conditioning rapidly deteriorates, increasing significantly after epoch 4. In contrast, \cmr effectively stabilizes and reduces the median $\kappa(G)$, maintaining a consistently lower (i.e., better) value throughout training. This clearly illustrates \cmr's ability to maintain well-conditioned layer operations.}
    \label{fig:median_kappa_appendix}
\end{figure}

\paragraph{Analysis of Median Conditioning.}
Figure~\ref{fig:median_kappa_appendix} provides a deeper insight into the spectral dynamics. The vanilla model's median layer condition number (blue line) exhibits a sharp increase after approximately 4 epochs, indicating a significant worsening of spectral properties during training. Conversely, \cmr (orange line) consistently maintains a much lower and more stable median condition number. This direct control over the median $\kappa(G)$ confirms \cmr's role in guiding the optimization towards well-conditioned parameter spaces, preventing the accumulation of ill-conditioning that can hinder training stability and efficiency.

\vspace{1cm} 

\begin{figure}[H]
    \centering
    \includegraphics[width=0.8\linewidth]{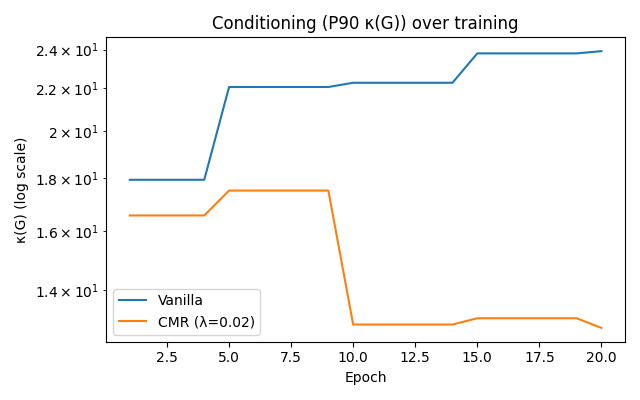} 
    \caption{\textbf{Evolution of 90th-Percentile Layer Condition Number ($\kappa(G)$) Over Training.} This plot, on a logarithmic scale, shows the 90th-percentile condition number, highlighting the behavior of the most ill-conditioned layers. Similar to the median, the vanilla baseline experiences a substantial increase in its P90 $\kappa(G)$, indicating that a significant portion of its layers become severely ill-conditioned. \cmr, however, effectively mitigates this rise, keeping the P90 $\kappa(G)$ remarkably stable and low, demonstrating robust control over even the "worst-case" layers in the network.}
    \label{fig:p90_kappa_appendix}
\end{figure}

\paragraph{Analysis of 90th-Percentile Conditioning.}
Figure~\ref{fig:p90_kappa_appendix} further reinforces these findings by focusing on the 90th-percentile condition number, which is a strong indicator of the most problematic layers within the network. The vanilla model again shows a dramatic increase, meaning that not just the average layer, but a significant fraction of its layers become severely ill-conditioned. In stark contrast, \cmr effectively suppresses this increase, maintaining a remarkably stable and low P90 $\kappa(G)$. This demonstrates that \cmr's regulatory effect is comprehensive, preventing critical degradation in conditioning even for the layers that are most prone to becoming ill-conditioned, thereby ensuring overall network stability.

\end{document}